\documentclass{article}

\usepackage{microtype}
\usepackage{graphicx}
\usepackage{subcaption}
\usepackage{booktabs} % for professional tables
\usepackage{xcolor}
\usepackage{color}
\usepackage{wrapfig}

\usepackage{hyperref}
\usepackage[preprint]{icml2026}

\usepackage{booktabs}
\usepackage{multirow}
\usepackage{float}
\usepackage{makecell}
\usepackage{threeparttable}

\usepackage{amsmath}
\usepackage{amssymb}
\usepackage{mathtools}
\usepackage{amsthm}

\usepackage[capitalize, noabbrev]{cleveref}

\theoremstyle{plain}
\newtheorem{theorem}{Theorem}[section]
\newtheorem{proposition}[theorem]{Proposition}
\newtheorem{lemma}[theorem]{Lemma}

\theoremstyle{definition}

\theoremstyle{remark}

\usepackage[linesnumbered,ruled,vlined,algo2e]{algorithm2e}

\newcommand{\dlt}[1]{\textcolor{gray}{#1}}

\newcommand{\our}[1]{\texttt{3D-Pruner}}

\usepackage[textsize=tiny]{todonotes}

\icmltitlerunning{Exploring 3D Dataset Pruning}

\begin{document}
\twocolumn[\icmltitle{Exploring 3D Dataset Pruning} \icmlsetsymbol{equal}{*}

\begin{icmlauthorlist}
  \icmlauthor{Xiaohan Zhao}{mbzuai}
  \icmlauthor{Xinyi Shang}{mbzuai}
  \icmlauthor{Jiacheng Liu}{mbzuai}
  \icmlauthor{Zhiqiang Shen}{mbzuai}
\end{icmlauthorlist}

\icmlaffiliation{mbzuai}{VILA Lab, Department of Machine Learning, MBZUAI}

\icmlcorrespondingauthor{Zhiqiang Shen}{zhiqiang.shen@mbzuai.ac.ae}

\icmlkeywords{Machine Learning, ICML} \vskip 0.3in ]



\printAffiliationsAndNotice{} % no special notice (required even if empty)

\begin{abstract}
Dataset pruning has been widely studied for 2D images to remove redundancy and accelerate training, while particular pruning methods for 3D data remain largely unexplored. In this work, we study dataset pruning for 3D data, where its observed common long-tail class distribution nature make optimization under conventional evaluation metrics {\em Overall Accuracy (OA)} and {\em Mean Accuracy (mAcc)} inherently conflicting, and further make pruning particularly challenging. To address this, we formulate pruning as approximating the full-data expected risk with a weighted subset, which reveals two key errors: {\em coverage error} from insufficient representativeness and {\em prior-mismatch bias} from inconsistency between subset-induced class weights and target metrics. We propose {\em representation-aware subset selection} with per-class retention quotas for long-tail coverage, and {\em prior-invariant teacher supervision} using calibrated soft labels and embedding-geometry distillation. The retention quota also serves as a switch to control the {\em OA-mAcc} trade-off. Extensive experiments on 3D datasets show that our method can improve both metrics across multiple settings while adapting to different downstream preferences.
Our code is available on \href{https://github.com/XiaohanZhao123/3D-Dataset-Pruning}{Github}.
\end{abstract}

\section{Introduction}

Dataset pruning~\citep{paul2021deep,sener2017active}, also known as coreset selection, is a well-established approach for reducing redundancy in large-scale training sets. In 2D image classification, many data pruning strategies have been extensively studied, including geometry-based clustering~\citep{sinha2019variational}, gradient matching~\citep{mirzasoleiman2020coresets_gradmatch1,killamsetty2021grad_gradmatch2}, and error-based importance estimation~\citep{paul2021deep,toneva2018empirical}. These methods can significantly accelerate training and reduce computational cost, and have been adopted in downstream scenarios such as hyper-parameter tuning~\cite{yang2022dataset} and continual learning~\cite{hao2023bilevel}. More recently, dataset pruning has also been combined with knowledge distillation~\cite{chen2025medium,ben2024distilling}, where a teacher trained on the full dataset guides both subset selection and post-pruning optimization.

% --- FIGURE: ShapeNet55 Distribution ---
\begin{figure}[t]
\centering
\includegraphics[width=0.99\linewidth]{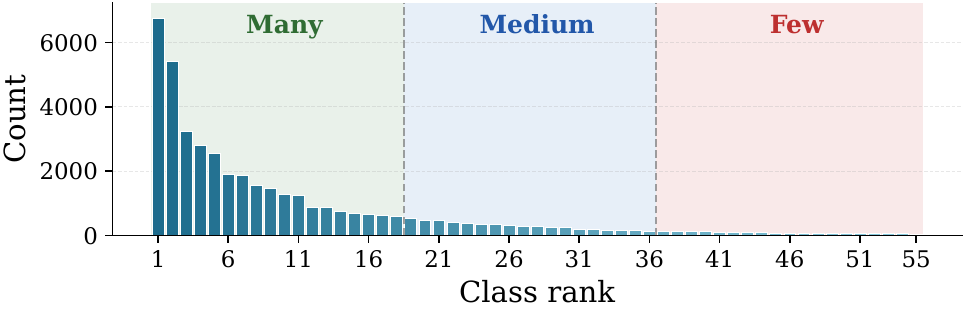}
\vspace{-1mm} 
\captionof{figure}{Grouped class distribution of ShapeNet55.}
\label{fig:shapenet_dist}
\vspace{1mm} 

\captionof{table}{Statistics of 3D imbalance on \textit{both} train and test set.} \label{tab:3d statistics}
\label{tab:dataset_stats}
\vspace{-3mm} 
\resizebox{0.99\linewidth}{!}{%
  \begin{tabular}{l cccc cccc}
    \toprule
    & \multicolumn{4}{c}{Train Set} & \multicolumn{4}{c}{Test Set} \\
    \cmidrule(lr){2-5} \cmidrule(lr){6-9}
    Dataset & Total & Max & Min & Ratio & Total & Max & Min & Ratio \\
    \midrule
    ScanObjectNN & 11,416 & 1,585 & 298 & 5.3$\times$   & 2,882  & 390   & 83 & 4.7$\times$ \\
    ModelNet40   & 9,840  & 889   & 64  & 13.9$\times$  & 2,468  & 100   & 20 & 5.0$\times$ \\
    ShapeNet55   & 41,952 & 6,747 & 44  & 153.3$\times$ & 10,518 & 1,687 & 12 & 140.6$\times$ \\
    \bottomrule
  \end{tabular}%
}
\vspace{-1.5em}
\end{figure}

Despite this progress, dataset pruning for 3D data remains unexplored as 3D datasets are expensive to build, and their class frequencies are rarely balanced by design. In practice, 3D data are often collected through manual modeling, such as CAD repositories, or through real-world scanning, so both training and test sets naturally inherit long-tailed distributions shaped by object frequency or dataset construction preferences. As a result, severe class imbalance is common in 3D benchmarks (see Fig.~\ref{fig:shapenet_dist} and Tab.~\ref{tab:3d statistics}), making pruning substantially more challenging than in standard 2D settings.

A key difficulty in 3D pruning is the tension between two widely used evaluation metrics: mean class accuracy (mAcc) and overall accuracy (OA). mAcc measures balanced performance across classes, while OA reflects expected utility under the naturally imbalanced test distribution. In 3D benchmarks, OA should not be viewed merely as a biased alternative to mAcc, when the test set is itself long-tailed, OA is the empirical estimate of performance under real query frequencies. Therefore, a practical pruning method for 3D data should account for both metrics, rather than optimizing only one at the expense of the other.

We argue that this challenge should not be addressed by making an early trade-off between OA and mAcc. Instead, the first step is to identify principles that are robust across different target priors, and only then introduce mechanisms for preference adjustment. To this end, we formulate pruning as a quadrature approximation of population risk and decompose the resulting error into two terms: \textit{representation error}, which captures how well the selected subset covers the underlying data manifold, and \textit{prior-mismatch bias}, which arises when the class distribution induced by the pruned subset differs from the distribution implied by the target evaluation metric. This perspective shows that no single subset is optimal for all priors, but also reveals shared optimization directions that can improve performance under both OA and mAcc.

Based on our theoretical analysis, we propose \texttt{3D-Pruner}. To reduce prior-mismatch bias, we decouple prior-robust class-conditional structure from prior-dependent offsets, and transfer this structural information through knowledge distillation. Specifically, we use calibrated soft targets together with geometry- and relation-preserving distillation to make supervision less sensitive to class-frequency bias. To reduce representation error, we analyze different selection signals under imbalanced 3D data and find that classifier-derived scalar scores are strongly correlated with class size, while embedding geometry is substantially more stable. We therefore perform subset selection based on representation coverage in the embedding space, together with a small per-class safety quota to preserve long-tail coverage and secure a strong shared performance floor across priors. Finally, since different downstream applications may prefer different evaluation priors, we introduce a lightweight steering wrapper that interpolates between stratified seeding and global selection, enabling flexible control over the trade-off between mAcc and OA for better downstream performance.

Our contributions are summarized as follows: 

\textbf{(i)} We identify a core challenge in 3D dataset pruning under long-tailed train/test splits: OA and mAcc reflect different yet practically important evaluation priors, making effective pruning fundamentally difficult. \\
\textbf{(ii)} We formulate pruning as quadrature over population risk and decompose its error into representation error and prior-mismatch bias, providing a prior-robust perspective that applies across different target priors. \\
\textbf{(iii)} Based on the analysis, we propose \texttt{3D-Pruner}, a 3D data pruning framework that improves the performance through a simple steering wrapper. To the best of our knowledge, this is the \textit{\textbf{first}} principled study of 3D dataset pruning.

\begin{figure*}[t]
\centering
\includegraphics[width=0.96\linewidth]{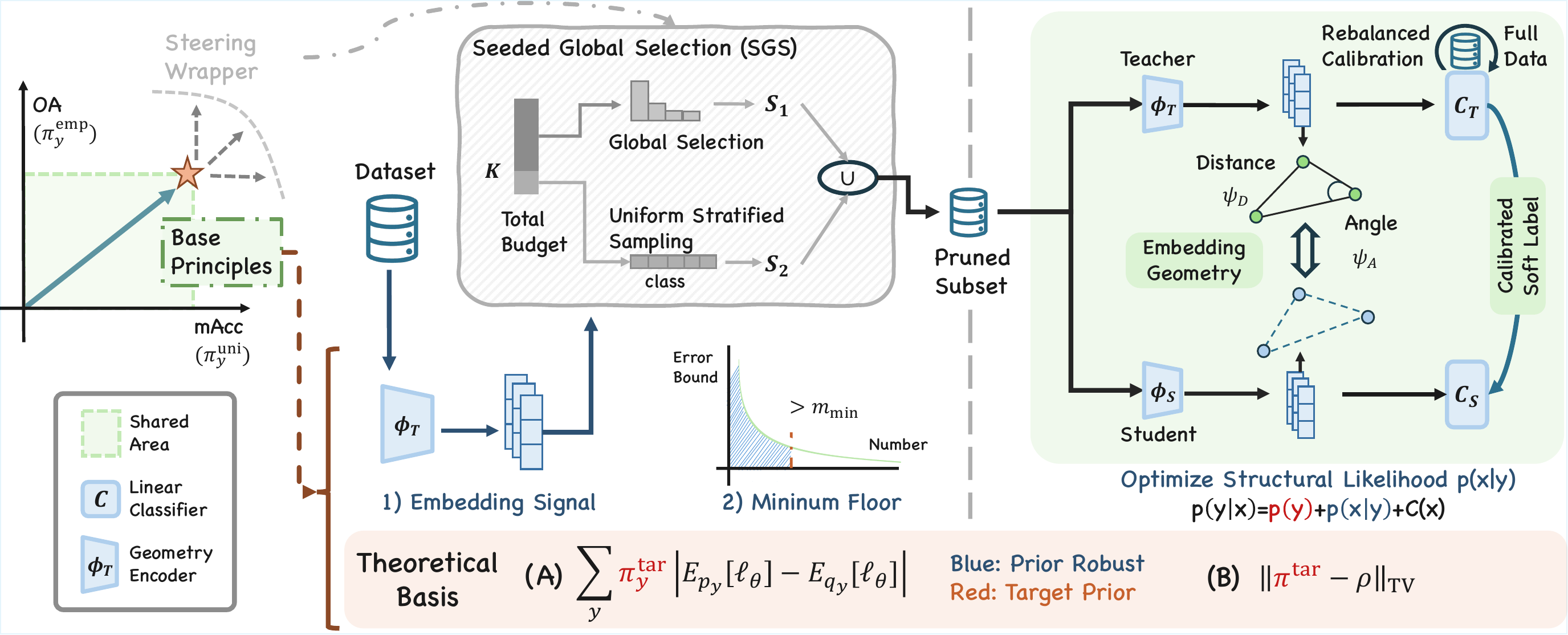}
\caption{Illustration of our \our{} framework, which comprises: (1) base principles (utilizing embedding signals, minimum floor selection, and optimizing structural likelihood in post-pruning training) that remain robust and beneficial across different priors, derived from theoretical analysis of the shared region, and (2) a steering warper that balances between the two priors.}
\label{fig:main_alg}
\end{figure*}

\section{Background} \label{background}

\textbf{3D Datasets.}
3D recognition benchmarks are essential for geometric learning, yet such data are costly to acquire and curate. Widely used datasets derive from two pipelines: curated CAD repositories (e.g., ModelNet40, ShapeNet~\citep{wu20153d_modelnet,chang2015shapenet,yi2016scalable_shapenet55}) and real-world scanning with depth sensors (e.g., ScanObjectNN~\citep{uy2019revisiting_scanobject}). Both share a common property: class frequencies are shaped by object availability and annotation cost rather than deliberate control, leading to long-tail imbalance in both training and evaluation splits (Fig.~\ref{fig:shapenet_dist}, Tab.~\ref{tab:3d statistics}).

This setting motivates reporting two complementary metrics driven by different \emph{evaluation priors}. Let $C$ be the number of classes, $N$ the total test examples, and $n_y$ the test examples in class $y$. Define per-class accuracy $\mathrm{Acc}_y:=\frac{1}{n_y}\sum_{i:y_i=y}\mathbf{1}\{\hat{y}_i=y_i\}$. Then: \\
\textbf{\textit{Overall accuracy (OA)}} with empirical prior $\pi^{\mathrm{emp}}_y=n_y/N$:
\[
\text{OA}=\sum_y \pi^{\mathrm{emp}}_y\,\mathrm{Acc}_y.
\]
OA reflects \textit{frequency as utility}: users predominantly query common objects (e.g., doors, tables), so high OA indicates reliability for daily use with presumably higher tolerance of rare cases like a strange ancient vase, where the system's value is determined by the cumulative success rate of user interactions. \\
\textbf{\textit{Mean class accuracy (mAcc)}} with uniform prior $\pi^{\mathrm{uni}}_y=\frac{1}{C}$:
\[
\text{mAcc}=\sum_y \pi^{\mathrm{uni}}_y\,\mathrm{Acc}_y.
\]
mAcc reflects \textit{capability versatility}: every category is equally important regardless of prevalence. High mAcc certifies discriminative features for all defined concepts, not merely exploiting frequency priors. Importantly, OA here should not be conflated with dense prediction settings where a dominant ``background'' label inflates overall scores. In 3D classification benchmarks, head classes correspond to real object categories that appear more often in scans or are more frequently modeled to meet market preference, yet still exhibit substantial intra-class variation. Thus, 3D models routinely report both metrics as equally legitimate but differently biased targets~\citep{qi2017pointnet++,qian2022pointnext,deng2023pointvector}. This duality is central for pruning, since subset selection can implicitly change the effective class prior and shift the balance between these two objectives.

\textbf{Data Pruning.} Dataset pruning, also termed as coreset selection, aims to identify a compact subset of training examples that preserves most of the learning signal of the full dataset, thereby reducing training cost. Existing methods span coverage and diversity-based selection, such as herding, k-center style objectives, and submodular subset selection~\cite{welling2009herding,sener2017active_kcenter,wei2015submodularity}, as well as gradient-based matching that approximates the optimization trajectory induced by the full data~\cite{mirzasoleiman2020coresets_gradmatch1,killamsetty2021grad_gradmatch2}. Another line of work ranks examples using training dynamics or error statistics and prunes those deemed redundant or uninformative~\cite{toneva2018empirical,paul2021deep_el2n}. Recently, teacher signals from knowledge distillation~\cite{hinton2015distilling}, especially soft labels, have also been incorporated into data pruning and subset training, improving the robustness and performance of pruned learning~\cite{hinton2015distilling,ben2024distilling,chen2025medium}. While prior studies are largely centered on 2D benchmarks, pruning in 3D often faces class imbalance, where the effect of pruning depends on the evaluation prior. Several recent methods address imbalance-aware pruning by enforcing class-wise stratified sampling with class-quota estimated by class-wise difficulty score (through a hold-out set performance or class-wise aggregation on scalar importance values like EL2N), but typically optimizing a fixed target criterion such as mean per-class accuracy or worst-class accuracy~\cite{vysogorets2024drop,zhang2025non,tsai2025class}. In contrast, we explicitly study pruning under two priors, the natural distribution and the mean-class distribution, and explicitly characterize their confidant and overlap in the 3D setting.

\section{Theoretical Analysis} \label{sec:theory}
To characterize the conflict between OA and mAcc, we formalize dataset pruning as a quadrature approximation of the population risk. This formulation enables an explicit decomposition of learning error into representation-driven and prior-driven components, thereby exposing the mathematical origin of this conflict.

\textbf{Pruning as quadrature approximation.}
Let $p$ denote the target distribution over examples $x$ (for brevity, $x$ may include an input--label pair), and define the population risk
$\mathcal{L}(\theta)=\mathbb{E}_{x\sim p}[\ell_{\theta}(x)]$.
Given a subset $S\subset\{1,\dots,n\}$ of size $m$ and weights $w\in\Delta^{m-1}$, define the discrete measure
$q_{S,w}=\sum_{i\in S}w_{i}\delta_{x_i}$ with $\delta$ denoting the Dirac function. Therefore, we can think of pruning as using a quadrature rule to approximate such integration, under the following form:
\begin{equation}
\hat{\mathcal{L}}_{S,w}(\theta)
=\mathbb{E}_{x\sim q_{S,w}}[\ell_{\theta}(x)]
=\sum_{i\in S}w_{i}\ell_{\theta}(x_{i}).
\end{equation}
The following lemma links the quality of this approximation directly to the generalization gap of the trained model.
\begin{lemma}[Generalization gap via discrepancy]\label{lem:gen_gap}
Let $\mathcal{G}=\{\ell_{\theta}(\cdot):\theta\in \Theta\}$ and define the discrepancy
$D_{\mathcal{G}}(p,q)=\sup_{g\in\mathcal{G}}|\mathbb{E}_{p}[g]-\mathbb{E}_{q}[g]|$.
If $\theta^{*}\in\arg\min_{\theta}\mathcal{L}(\theta)$ and $\hat{\theta}\in\arg\min_{\theta}\hat{\mathcal{L}}_{S,w}(\theta)$, then
\begin{equation}
  \mathcal{L}(\hat{\theta})-\mathcal{L}(\theta^{*})
  \le 2\,D_{\mathcal{G}}(p, q_{S,w}),
  \label{eq:upper_bound}
\end{equation}
where $D_{\mathcal{G}}$ is an instance of an integral probability metric (IPM) \citep{muller1997integral}.
\end{lemma}
Given labels are imbalanced, we further decompose the bound class-wise.
Assume the target distribution admits a label mixture
$p=\sum_{y}\pi^{\mathrm{tar}}_{y}\,p_{y}$,
where $p_{y}$ is the class-conditional distribution and $\pi^{\mathrm{tar}}$ is the \emph{target} class prior (e.g., the evaluation prior $\pi_y^\mathrm{uni}$ or $\pi_y^{\mathrm{emp}}$).
Partition $S=\cup_{y}S_{y}$ with $|S_{y}|=m_{y}$ and write
\begin{equation}
\begin{aligned}
  & q_{S,w}=\sum_{y}\rho_{y}q_{y},\qquad
  \rho_{y}:=\sum_{i\in S_y}w_{i},\qquad \\
  & q_{y}:=\rho_{y}^{-1}\sum_{i\in S_y}w_{i}\delta_{x_i}\ \ (\rho_y>0).
\end{aligned}
\end{equation}
Here $\rho$ is the class distribution \emph{induced} by the weighted pruned objective.
Under mild regularity conditions (e.g., bounded loss $\|\ell_{\theta}\|_{\infty}\le B$), for any fixed $\theta$,
\begin{equation}
\begin{aligned}
  \big|\mathbb{E}_{p}[\ell_{\theta}]-\mathbb{E}_{q_{S,w}}[\ell_{\theta}]\big|
  \le\ &\underbrace{\sum_y \pi^{\mathrm{tar}}_y
  \big|\mathbb{E}_{p_y}[\ell_\theta]-\mathbb{E}_{q_y}[\ell_\theta]\big|}_{\text{(A) representation error}} \\
  &+\underbrace{2B\,\|\pi^{\mathrm{tar}}-\rho\|_{\mathrm{TV}}}_{\text{(B) prior mismatch bias}}.
\end{aligned}
\label{eq:decomp}
\end{equation}
Term (A) measures the approximation quality within each class, while term (B) depends only on how far the induced prior $\rho$ is from the target prior $\pi^{\mathrm{tar}}$. We next analyze the two terms. For term (A), many selection rules yield a per-class approximation rate of the following form.

\begin{lemma}[Per-class approximation rate (informal)]\label{lem:class_rate}
With high probability (given $m_y$),
\begin{equation}
  \sup_{\theta\in\Theta}\big|\mathbb{E}_{p_y}[\ell_{\theta}] - \mathbb{E}_{q_y}[\ell_{\theta}]\big|
  \ \lesssim\ \frac{c_{y}}{m_{y}^{\gamma}} + \mathrm{BiasTerm}_{y},
  \label{eq:alloc}
\end{equation}
where $c_{y}$ captures the data/model complexity of class $y$ (we refer to it as \emph{class complexity}).
The exponent $\gamma$ depends on the selection strategy; uniform random sampling typically gives $\gamma=\tfrac{1}{2}$.
\end{lemma}

\begin{theorem}[Optimal allocation for term (A)]\label{thm:opt_alloc}
Assume the rate exponent $\gamma$ is the same across classes and ignore rounding effects.
The allocation that minimizes
$\sum_y \pi^{\mathrm{tar}}_y \big|\mathbb{E}_{p_y}[\ell_\theta]-\mathbb{E}_{q_y}[\ell_\theta]\big|$
under a fixed budget $\sum_y m_y=m$ satisfies
\begin{equation}
  m_y \propto (c_{y}\pi_{y}^{\mathrm{tar}})^{k},
  \qquad k=\frac{1}{1+\gamma}.
\end{equation} \label{equ:opt_alloc}
\end{theorem}
\vspace{-0.1in}
This shows that the representation-optimal subset is generally not balanced.
Classes with larger $c_y$ (harder classes) and/or larger $\pi_y^{\mathrm{tar}}$ receive more budget.
This dependence on $\pi^{\mathrm{tar}}$ is important: different choices of the target prior emphasize different allocations.

For term (B), class-wise reweighing in post-pruning training can resolve. One choice is $w_{i}=\pi^\mathrm{tar}_{y}/{m_{y}} \quad \text{for } i\in S_{y}$,
which ensures $\rho_{y}=\pi^{\mathrm{tar}}_{y}$ for all $y$ and eliminates term (B) in~\eqref{eq:decomp}.
The remaining bound depends only on representation:
\begin{equation}
\mathrm{Error}(\pi^{\mathrm{tar}})
\ \lesssim\ \sum_{y} \pi_y^{\mathrm{tar}}\left(\frac{c_{y}}{m_{y}^{\gamma}}+\mathrm{BiasTerm}_{y}\right).
\label{eq:optimal}
\end{equation}
The optimal allocation of $\{m_y\}$ follows Equ.~(\ref{equ:opt_alloc}).

\noindent
\textbf{Dilemma.}
The analysis above applies to any target prior $\pi^{\mathrm{tar}}$.
In this work, we focus on a uniform prior $\pi^{\mathrm{uni}}$ and an empirical prior $\pi^{\mathrm{emp}}$. Our initial analysis shows that a subset optimized for $\pi^{\mathrm{uni}}$ will, in general, differ from one optimized for $\pi^{\mathrm{emp}}$. In short, the optimal condition derived from (A) and (B) reflects an inherent conflict between the two priors across selection and the post-pruning training.

\section{Method}

\paragraph{Roadmap.} To tackle the dilemma in Sec~\ref{sec:theory}, we prioritize establishing base principles that are robust and beneficial across priors, then applies steering wrapper, as shown in Fig~\ref{fig:main_alg}. Guided by our theoretical analysis, the method unfolds in three parts: first, we target the prior mismatch bias (Term B) by decoupling structural likelihood from class priors during post-training. Second, we minimize the representation error (Term A) by identifying geometric pruning signals robust to imbalance with minimum floor. Finally, we introduce a steering wrapper that explicitly modulates the selection strategy to satisfy different user preferences.

\begin{figure*}[t]
\centering
\includegraphics[width=0.98\linewidth]{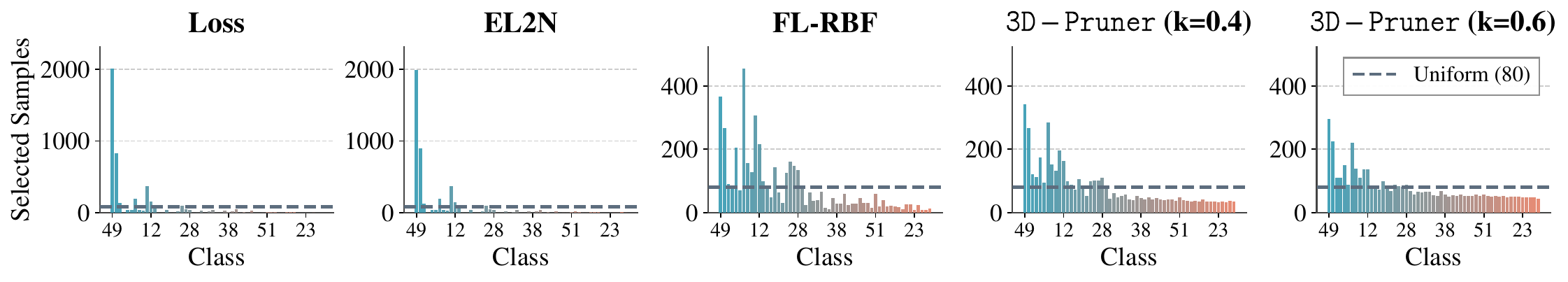}
\vspace{-0.3cm}
\caption{Selection composition across algorithms. Incomparable scalar scores (loss, EL2N) bias selection toward many-shot classes. Geometric embedding selection is more robust but still partially affected by class distribution, under-selecting certain classes. SGS mitigates this, also enabling a flexible balancing.}
\label{fig:composition}
\vspace{-1.5em}
\end{figure*}

\subsection{Resolving Term~B: Robust Post-pruning Distillation}

\textbf{Decomposing the Dilemma}\label{subsec:decomposition}
Hard-label supervision conflates two qualitatively different factors, i.e., the class-conditional structure and the class prevalence. Following Bayes' rule:
\begin{equation}
\underbrace{\log p(y\mid x)}_{\text{Posterior}}
=
\underbrace{\log p(x\mid y)}_{\text{Structural Likelihood}}
+
\underbrace{\log p(y)}_{\text{Class Prior}}
+
C(x),
\label{eq:posterior_decomp}
\end{equation}
where $C(x)$ is independent of $y$.
Reweighting simply scales the posterior, implicitly forcing the model to fit a distorted prior. Crucially, in this paradigm, the pruning weights $w$ \emph{define} the target distribution; thus, any slight misalignment in $w$ shifts the global optimum.
However, we observe that $\log p(x\mid y)$ describes the semantic geometry of the data manifold, which is \emph{shared} regardless of the evaluation metric (OA or mAcc). The trade-off is fundamentally a conflict of priors, not structures.

\textbf{Decoupling via Knowledge Distillation.}
To explicitly learn the shared structure separate from the prior offset, we employ Knowledge Distillation (KD).
While soft labels are standard in dataset distillation for condensing information density~\cite{yin2023squeeze,cui2023scaling}, we exploit a distinct, typically overlooked property: KD's \emph{structural invariance} to reweighting.
Consider the KD objective on a pruned subset $S$ with weights $w$:
\begin{equation}
\hat{\mathcal{L}}_{\mathrm{KD}}^{S,w}(\theta)
=
\sum_{i\in S} w_i \,
\mathrm{CE}\ \big(T(\cdot\mid x_i),\, f_\theta(\cdot\mid x_i)\big).
\label{eq:kd_emp_main}
\end{equation}
Crucially, because the teacher's posterior $T(\cdot|x_i)$ is fixed instance-wise, the weights $w$ only modulate the \emph{optimization pace} of each sample, but never alter the \emph{optimal target} itself.
This stands in sharp contrast to hard-label supervision, where $w$ constitutes the target prior. KD renders the learning objective theoretically robust to the reweighting.

\begin{proposition}[Weight-robustness of KD]\label{prop:kd_weight_robust}
Assume the student is expressive enough to interpolate the teacher on the support $S$, namely
$f_\theta(\cdot\mid x_i)=T(\cdot\mid x_i)$ for all $i\in S$.
Then for any strictly positive weights $w$, any interpolating solution is a global minimizer of
$\hat{\mathcal{L}}_{\mathrm{KD}}^{S,w}$.
\end{proposition}

\textit{Insight.}
KD effectively shifts the bottleneck from tuning the weights $w$ to calibrating the teacher $T$. By fixing the reference for the structural likelihood, the student preserves the correct manifold geometry regardless of the re-balancing.

\begin{table}[ht]
\centering
\vspace{-0.5em}
\scriptsize
\setlength{\tabcolsep}{3.5pt}
\renewcommand{\arraystretch}{1.05}
\caption{\textbf{KD improves accuracy and reduces sensitivity to rebalancing.}
  Rebalancing policies include instance-balanced sampling (IB), class-balanced sampling (CB),
and square-root sampling (Sqrt)~\cite{kang2019decoupling}, plus class-balanced loss (CB-Loss)~\cite{cui2019class}. Results are obtained on ShapeNet55 with PointNet++, with imbalanced set selected by FL-RBF~\cite{wei2015submodularity}.}
\label{tab:kd_rebalancing_robust}
\vspace{-2mm}
\begin{tabular}{@{}lcccc@{}}
  \toprule
  & \multicolumn{2}{c}{Hard} & \multicolumn{2}{c}{KD} \\
  \cmidrule(lr){2-3}\cmidrule(lr){4-5}
  Rebalancing & OA & mAcc & OA & mAcc \\
  \midrule
  CB-Samp.\ (CB)     & 79.02 & 61.20 & 82.09 & 64.56 \\
  CB-Loss            & 77.45 & 62.16 & 81.94 & 63.57 \\
  IB-Samp.\ (IB)     & 79.92 & 58.41 & 82.04 & 64.02 \\
  Sqrt-Samp.\ (Sqrt) & 80.00 & 60.10 & 81.21 & 64.81 \\
  \midrule
  \textbf{Mean$\pm$Std} &
  \textbf{79.10$\pm$1.18} & \textbf{60.47$\pm$1.61} &
  \textbf{81.82$\pm$0.41} & \textbf{64.24$\pm$0.56} \\
  \bottomrule
\end{tabular}
\vspace{-0.6cm}
\end{table}

\paragraph{Optimizing the Structure.}\label{subsec:teacher_calib}
In 3D recognition, standard protocols typically minimize average error without rebalancing. Consequently, a teacher trained under this regime inevitably encodes the long-tailed prior, producing biased targets that conflate high-quality structural signals with class prevalence. To extract a refined \emph{structural likelihood}, we employ a twofold refinement targeting both boundary alignment and manifold geometry. \\
\textit{1) Calibrating the Structural Source (Boundary Alignment).}
First, we mitigate the prior bias in the teacher's logits to correct the sample-to-boundary relations.
Recognizing that the learned embedding $\phi$ already captures robust semantic geometry~\citep{kang2019decoupling}, we freeze $\phi$ and re-train only the classifier head $(W,b)$ using a class-balanced objective on the full dataset:
\begin{equation}
\min_{W,b}
\mathbb{E}_{p_{\mathrm{train}}}
\Big[\alpha_y \, \mathrm{CE}\big(\delta_y, \sigma(W\phi(x)+b)\big)\Big].
\label{eq:teacher_calib_obj}
\end{equation}
This adjustment neutralizes the head's tendency to favor majority classes, aligning the logits closer to the intrinsic likelihood $p(x|y)$. \\
\textit{2) Preserving Geometry under Sparsity (Intrinsic Topology).}
While calibration refines the decision boundaries, pruning introduces a density problem: the manifold support collapses into sparse anchor points.
Recognizing that structural information comprises not only the \emph{extrinsic} distance to decision boundaries (via logits) but also the \emph{intrinsic} relational topology among samples, we explicitly transfer the latter using Relational Knowledge Distillation (RKD)~\citep{park2019relational}.
RKD enforces consistency in pairwise distances ($\psi_D$) and triplet angles ($\psi_A$) within a batch $\mathcal{B}$:
\begin{equation}
\mathcal{L}_{\mathrm{RKD}}
= \sum_{\mathcal{B}^2} \ell_{\delta}(\psi_D^T, \psi_D^S) + \sum_{\mathcal{B}^3} \ell_{\delta}(\psi_A^T, \psi_A^S).
\label{eq:rkd_loss}
\end{equation}
By anchoring the relative geometry, RKD enables the student to reconstruct the \textit{internal shape} of the class manifold even from sparse samples, effectively compensating for the information loss in Term~B.

\subsection{Resolving Term~A: Geometry-aware Selection} \label{sec:robust_selection_shared_geometry}

\textbf{1) Robust Signals: Leveraging 3D Inductive Bias.}
The first challenge in selection is identifying a pruning metric that remains comparable across highly imbalanced classes. We believe the importance of signal quality outweighs that of further allocation based on the signal. We identify \emph{embedding geometry} as the robust signal for 3D data, distinct from classifier-derived scalars (e.g., Loss, Entropy, EL2N).

In 3D recognition, models rely on deep inductive biases to learn local geometric primitives (e.g., corners, flats, curvatures). Crucially, these geometric units are \emph{shared} across both few-shot and many-shot classes, making the embedding space inherently more stable and comparable than the classifier decision boundaries.
In contrast, scalar signals are entangled with the training prior: as shown in Table~\ref{tab:signal_audit}, metrics like Loss and EL2N exhibit extreme dependence on class frequency, collapsing selection onto head classes (as shown in Fig~\ref{fig:composition}).
Furthermore, selecting based on embedding geometry aligns perfectly with our Term~B optimization, which explicitly preserves this structure via RKD.
\begin{table}[ht]
\vspace{-1.5em}
\centering
\caption{\textbf{Signal Audit.} We measure correlation ($\rho$) and dependence ($R^2$) of class-aggregated signal magnitude on class size, and head/tail score overlap. For embeddings, magnitude is distance to the class center. \textbf{Embedding signals show minimal class-frequency dependence and high cross-class overlap}, whereas scalar signals (Loss, EL2N) are dominated by the training prior, causing extreme selection imbalance (up to 40$\times$).}
\label{tab:signal_audit}
\vspace{-0.4em}
\resizebox{1.0\linewidth}{!}{
\begin{tabular}{lcccc}
\toprule
\textbf{Signal} & \textbf{$\rho$} (Corr. w/ Size) $\downarrow$ & \textbf{$R^2$} (Dep. on Size) $\downarrow$ & \textbf{Overlap} $\uparrow$ & \textbf{Imbalance Ratio} $\downarrow$ \\
\midrule
\textbf{Embedding} & \textbf{-0.306} & \textbf{0.148} & \textbf{96\%} & \textbf{1.88x} \\
Loss      & 0.635  & 0.145 & 26\% & 40.94x \\
EL2N      & 0.594  & 0.164 & 30\% & 33.01x \\
Entropy   & 0.731  & 0.481 & 26\% & 7.84x \\
\bottomrule
\end{tabular}%
}
\end{table}

\textbf{2) A Prior-Agnostic High-Return Regime.}
Even with a robust signal, pure global selection naturally drifts towards majority classes due to their density.
From Theorem~\ref{thm:opt_alloc}, we know a single fixed subset cannot be simultaneously optimal for divergent target priors (e.g., OA vs. mAcc).
However, the error bound reveals a \emph{shared high-return regime} that exists before any target-specific trade-off kicks in.

Let $E_y(m_y) \propto m_y^{-\gamma}$ denote the power-law error decay. The marginal gain from the first few samples is extremely steep. Specifically, allocating a minimal floor of $b$ samples captures a constant fraction of the reducible error:
\begin{equation}
\frac{E_y(1)-E_y(b)}{E_y(1)} = 1-\frac{1}{b^{\gamma}}.
\end{equation}
Crucially, this relative gain across all classes is \emph{independent} of any reweighting induced by the target prior $\pi^{\rm tar}$.
This motivates our strategy to reserve a safety budget to guarantee a minimum floor $m_y \ge b$ for every class.
By securing this ``safety floor," we harvest the common high-curvature region of the error surface, preventing systematic under-coverage of few-shot classes regardless of the final evaluation metric.

\textbf{Seeded Global Selection (SGS): Steering the Prior.}
Finally, we unify the proposed principles into a streamlined selection wrapper.
Having established that a robust subset requires both a \emph{safety floor} (to capture the shared high-return regime) and \emph{geometry-aware selection} (to leverage robust signals), SGS is designed as a direct instantiation of these two requirements.
It interpolates between a guaranteed floor and a data-driven distribution using a single steering parameter $K\in[0,1]$.

SGS operates in two modes, strictly following our theoretical analysis in Term~A.
\textit{Mode 1 (Seeding)} implements the \emph{safety floor} principle. It allocates a budget of $m_{\mathrm{seed}}=\lfloor K m\rfloor$ via stratified sampling, securing the base representation required for uniform-prior robustness (mAcc), also capturing the initial high-reward region.
\textit{Mode 2 (Global)} implements the \emph{geometry-aware} principle on the remaining budget ($m - m_{\mathrm{seed}}$). It runs global embedding selection to capture dense or complex regions naturally, which typically benefits OA.
The final subset is the union of both. By adjusting $K$, SGS steers the selection from a protective floor (high $K$) to a fully density-driven distribution (low $K$). This allows users to flexibly balance the preference between mAcc and OA without re-designing the selection metric or solving complex allocation problems.
Algorithm~\ref{alg:refined_selection} details the procedure.

\begin{algorithm2e}[t]
\DontPrintSemicolon
\SetKwInOut{Input}{Input}
\SetKwInOut{Output}{Output}
\caption{Refined Hybrid Selection}
\label{alg:refined_selection}

\Input{Dataset $\mathcal{D}$, Budget $B$, Safety ratio $k$, Classes $\mathcal{C}$,
Selector $\phi(\mathcal{X}, n, \mathcal{S}_{init})$ selecting $n$ items from $\mathcal{X}$ given seed $\mathcal{S}_{init}$}
\Output{Final Coreset $\mathcal{S}$}
\BlankLine

$b \leftarrow \lfloor (K \cdot B) / |\mathcal{C}| \rfloor$\;

$\mathcal{S}_{\rm strat} \leftarrow \bigcup_{c \in \mathcal{C}} \phi(\mathcal{D}_c, b, \emptyset)$,\;
$\mathcal{S}_{\rm glob} \leftarrow \phi(\mathcal{D}, B-K|C|, \emptyset)$\;

\BlankLine
$\mathcal{S}_{\rm union} \leftarrow \mathcal{S}_{\rm strat} \cup \mathcal{S}_{\rm glob}$\;

\If{$|\mathcal{S}_{union}| < B$}{
$\mathcal{S}_{\rm union} \leftarrow \phi(\mathcal{D}, B-|S_\mathrm{union}|, \mathcal{S}_{\rm union})$\;
}
\KwRet{$\mathcal{S}$}\;
\end{algorithm2e}

\begin{figure*}[t]
\centering
\begin{subfigure}[b]{0.495\linewidth}
\centering
\includegraphics[width=0.495\linewidth]{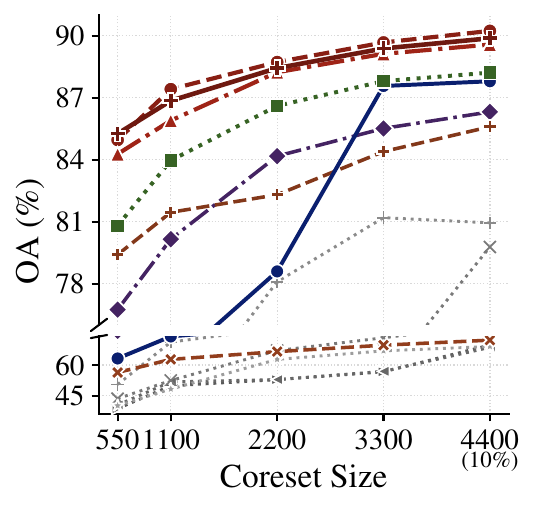}
\includegraphics[width=0.495\linewidth]{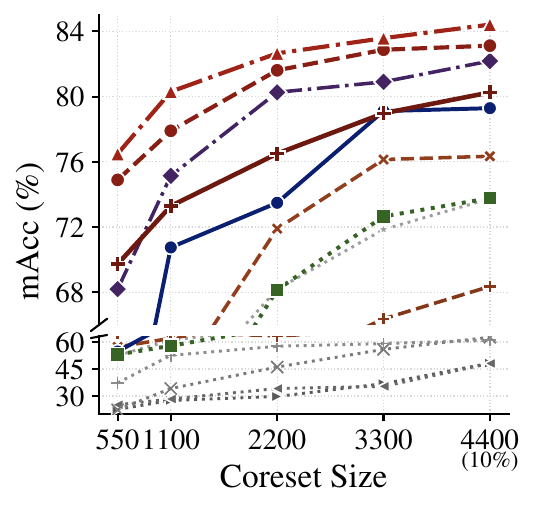}
\caption{PointNet++, ShapeNet}
\label{fig:shapenet}
\end{subfigure}
\begin{subfigure}[b]{0.495\linewidth}
\centering
\includegraphics[width=0.495\linewidth]{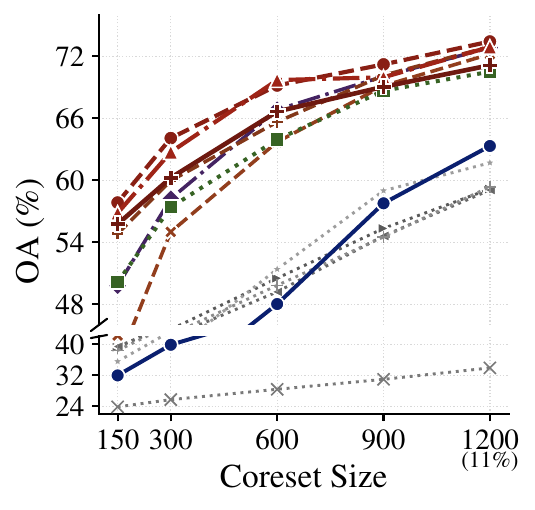}
\includegraphics[width=0.495\linewidth]{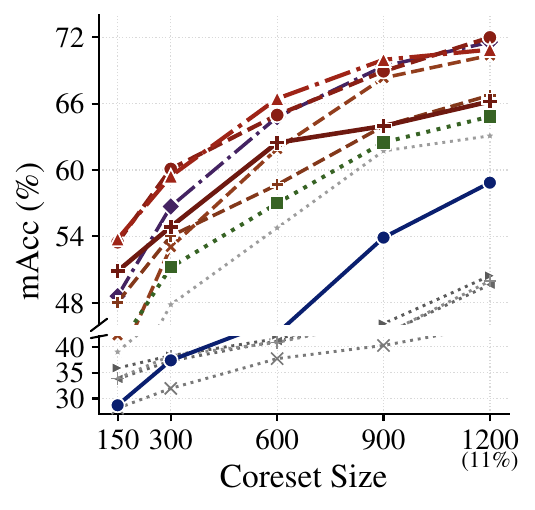}
\caption{PointNet++, ScanObjectNN}
\label{fig:scanobject}
\end{subfigure}

\begin{subfigure}[b]{0.495\linewidth}
\centering
\includegraphics[width=0.495\linewidth]{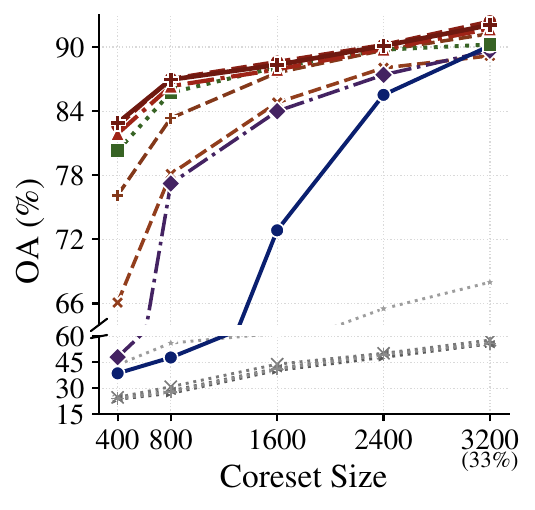}
\includegraphics[width=0.495\linewidth]{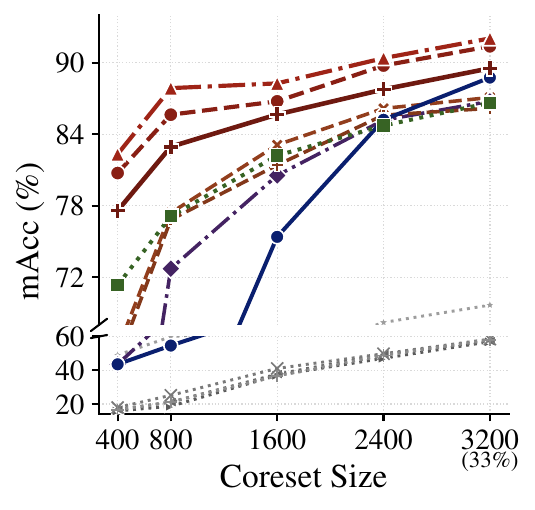}
\caption{PointMAE, ModelNet40}
\label{fig:pointmae}
\end{subfigure}
\begin{subfigure}[b]{0.495\linewidth}
\centering
\includegraphics[width=0.495\linewidth]{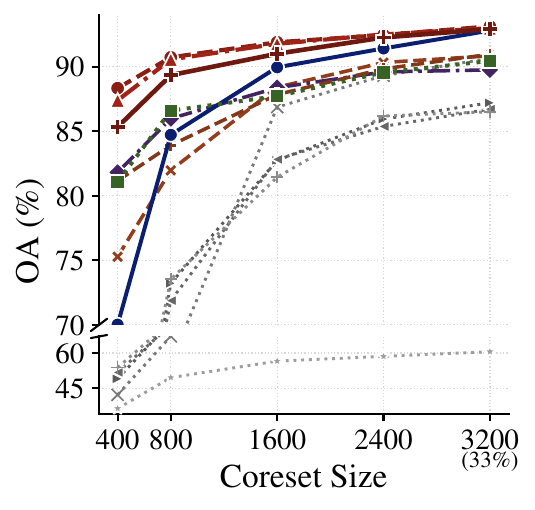}
\includegraphics[width=0.495\linewidth]{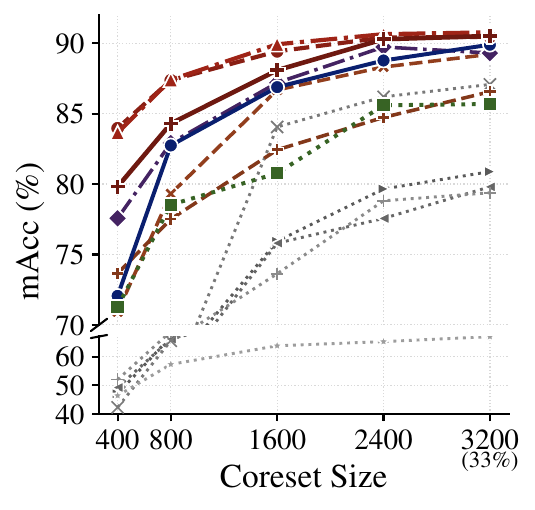}
\caption{PointNet++, ModelNet40}
\label{fig:modelnet}
\end{subfigure}

\vspace{-0.5em}
\begin{center}
\includegraphics[width=0.95\linewidth]{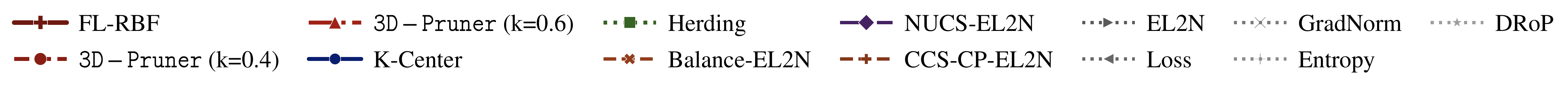}
\end{center}
\vspace{-1.0em}
\caption{Comparison of pruning methods across various datasets, models, and budgets on point cloud modality.}
\label{fig:main_experiment}
\vspace{-1.0em}
\end{figure*}

\section{Experiment}

\subsection{Experiment Setup}
\textbf{Datasets.} We evaluate on point clouds (ModelNet40~\citep{wu20153d_modelnet}, ScanObjectNN~\citep{uy2019revisiting_scanobject}, ShapeNet55~\cite{yi2016scalable_shapenet55}) and meshes (ModelNet40). \\
\textbf{Models.} For point clouds, we use CNN/MLP (PointNeXt~\citep{qian2022pointnext}, PointVector~\citep{deng2023pointvector}, PointNet++~\citep{qi2017pointnet++}) and transformer-based PointMAE~\citep{pang2023masked}. For meshes, we use MeshMAE~\cite{liang2022meshmae} and MeshNet~\cite{feng2019meshnet}. \\
\textbf{Methods.} We evaluate classical pruning metrics: loss, gradient norm, EL2N~\citep{paul2021deep_el2n}, entropy~\citep{wang2014new_entropy}, herding~\cite{welling2009herding}, K-center Greedy~\cite{sener2017active_kcenter}, FL-RBF~\citep{wei2015submodularity}, spling a variety of scalar difficulty based and embedding geometric based, as well as imbalance-aware methods: DRoP~\citep{vysogorets2024drop}, NUCS~\citep{zhang2025non}, and CCS-CB~\citep{tsai2025class}.\\
\textbf{Metrics \& Budget.} Following Sec.~\ref{background}, we report OA and mAcc per standard protocols. We focus on the high-compression regime to \emph{explicitly} distinguish dataset pruning from head-class downsampling~\cite{buda2018systematic}.
Given the extreme skew (e.g., in ShapeNet55: 84.9\% Many-shot vs. 1.6\% Few-shot), this setting isolates \textit{pure} selection efficacy under strict constraints
(extended results on relaxed budgets, additional results on different models, and detailed experimental settings provided in the appendix).

\begin{figure}[t]
\centering
\begin{subfigure}[b]{0.85\linewidth}
\centering
\includegraphics[width=\linewidth]{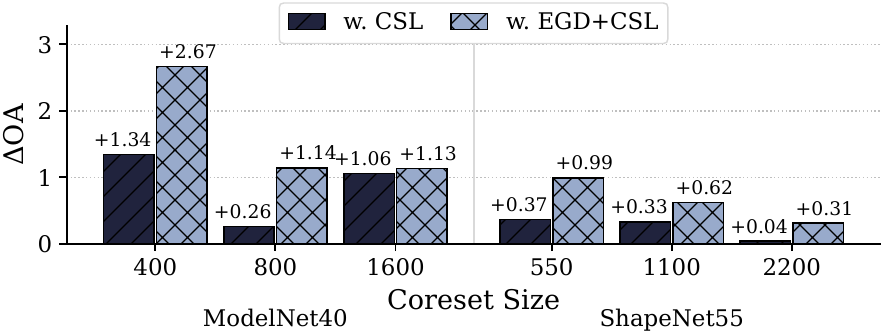}
\end{subfigure}

\begin{subfigure}[b]{0.85\linewidth}
\centering
\includegraphics[width=\linewidth]{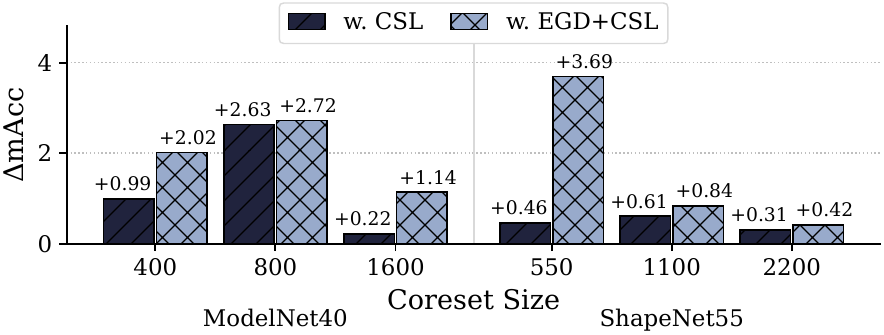}
\end{subfigure}
\caption{Impact of robust distillation methods: Calibrated Soft Label (CSL) and Embedding Geometry Distillation (EGD) on OA and mAcc. Results are obtained on PointNet++.} \label{fig:term B}
\vspace{-2.0em}
\end{figure}

\begin{figure}[t]
\centering
\centering
\includegraphics[width=0.495\linewidth]{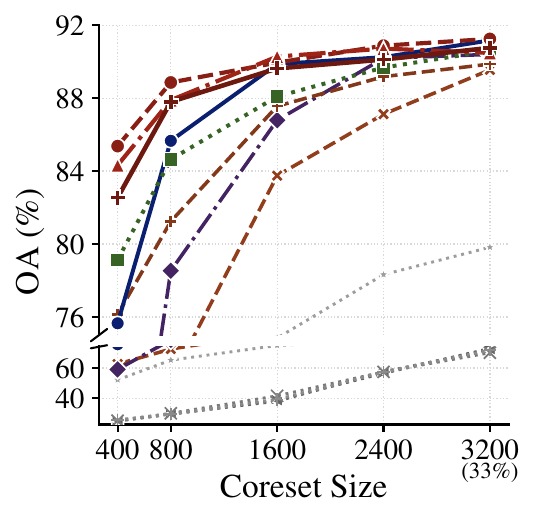}
\includegraphics[width=0.495\linewidth]{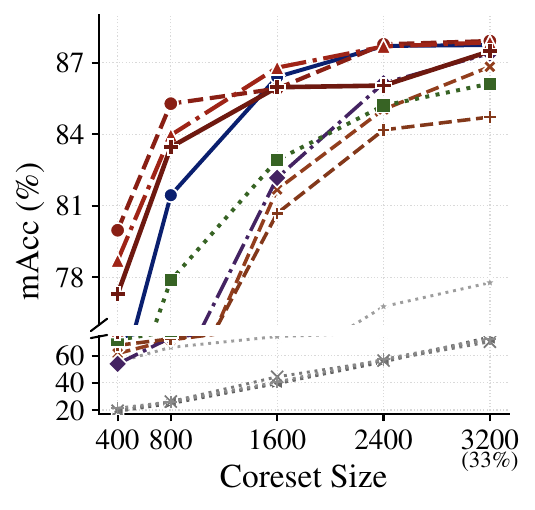}
\caption{Comparison of pruning methods on mesh modality, legend following Fig.~\ref{fig:main_experiment}}
\label{fig:meshnet}
\vspace{-2.0em}
\end{figure}

\begin{table*}[t]
\centering
\caption{\textbf{\textit{Fair comparison}} of selection strategies on \emph{ModelNet40}, \emph{ScanObjectNN}, and \emph{ShapeNet55} with Calibrated Soft Label and Embedding Geometry Distillation from \our{} applied to all baselines.}
\label{tab:fair comparison}
\vspace{-2mm}
\resizebox{\textwidth}{!}{
\begin{tabular}{ll cc cc cc cc cc cc cc cc cc}
\toprule
& & \multicolumn{6}{c}{\textbf{ModelNet40}} & \multicolumn{6}{c}{\textbf{ScanObjectNN}} & \multicolumn{6}{c}{\textbf{ShapeNet55}} \\
& & \multicolumn{2}{c}{$m=400$} & \multicolumn{2}{c}{$m=800$} & \multicolumn{2}{c}{$m=1600$}
& \multicolumn{2}{c}{$m=150$} & \multicolumn{2}{c}{$m=300$} & \multicolumn{2}{c}{$m=600$}
& \multicolumn{2}{c}{$m=550$} & \multicolumn{2}{c}{$m=1100$} & \multicolumn{2}{c}{$m=2200$} \\
\cmidrule(lr){3-4} \cmidrule(lr){5-6} \cmidrule(lr){7-8}
\cmidrule(lr){9-10} \cmidrule(lr){11-12} \cmidrule(lr){13-14}
\cmidrule(lr){15-16} \cmidrule(lr){17-18} \cmidrule(lr){19-20}
\textbf{Metric Type} & \textbf{Selection Policy} & OA & mAcc & OA & mAcc & OA & mAcc & OA & mAcc & OA & mAcc & OA & mAcc & OA & mAcc & OA & mAcc & OA & mAcc \\
\midrule

% -------------------------------------------
% SCALAR SCORES
% -------------------------------------------
\multirow{13}{*}{\shortstack{\textbf{Scalar}\\\textbf{Scores}}}
& \textit{Global Selection} & & & & & & & & & & & & & & & & & & \\
& \quad Loss     & 56.36 & 47.76 & 75.97 & 65.97 & 88.24 & 82.43 & 35.84 & 35.60 & 52.36 & 51.37 & 63.15 & 61.10 & 55.21 & 20.12 & 60.13 & 25.34 & 67.87 & 25.38 \\
& \quad GradNorm & 44.61 & 35.72 & 70.22 & 59.32 & 87.96 & 79.62 & 45.70 & 36.22 & 51.84 & 42.47 & 60.40 & 50.63 & 62.82 & 22.60 & 67.39 & 27.18 & 71.06 & 31.54 \\
& \quad Entropy  & 58.71 & 49.24 & 77.84 & 68.37 & 88.37 & 82.74 & 34.46 & 34.80 & 50.97 & 49.68 & 63.21 & 62.30 & 66.47 & 24.06 & 75.59 & 33.00 & 79.57 & 41.27 \\
& \quad EL2N     & 55.27 & 44.80 & 73.54 & 64.81 & 88.41 & 82.83 & 35.77 & 35.61 & 51.39 & 50.62 & 63.98 & 62.90 & 52.23 & 19.45 & 62.53 & 23.17 & 65.81 & 24.97 \\
\cmidrule{2-20}

& \textit{Stratified Sampling} & & & & & & & & & & & & & & & & & & \\
& \quad DRoP         & 41.86 & 45.02 & 55.63 & 57.47 & 64.82 & 65.32 & 38.06 & 39.77 & 46.29 & 48.54 & 58.39 & 58.71 & 45.50 & 54.73 & 54.79 & 60.72 & 64.45 & 69,40 \\
& \quad Balance-EL2N        & 74.39 & 69.39 & 83.23 & 81.43 & 89.82 & 88.70 & 49.34 & 49.14 & 59.02 & 56.92 & 65.68 & 64.02 & 78.53 & 74.45 & 82.89 & 79.26 & 84.46 & 81.79 \\
& \quad NUCS-EL2N        & 82.41 & 78.72 & 86.26 & 85.04 & 89.62 & 88.51 & 54.41 & 52.46 & 61.76 & 59.82 & 67.93 & 66.77 & 78.17 & 71.89 & 80.66 & 75.90 & 85.02 & 79.92 \\
& \quad CCS-CP-EL2N      & 81.44 & 74.55 & 85.49 & 79.71 & 89.34 & 82.39 & 55.59 & 48.89 & 61.69 & 54.87 & 68.25 & 62.12 & 80.90 & 70.30 & 83.81 & 66.43 & 84.81 & 67.80 \\

\midrule
\midrule

% -------------------------------------------
% VECTOR EMBEDDINGS
% -------------------------------------------
\multirow{8}{*}{\shortstack{\textbf{Vector}\\\textbf{Emb.}}}
& \textit{Global Selection} & & & & & & & & & & & & & & & & & & \\
& \quad Herding  & 81.89 & 74.25 & 87.70 & 80.68 & 89.10 & 83.84 & 50.21 & 44.63 & 59.54 & 53.77 & 66.58 & 61.11 & 82.21 & 60.02 & 83.99 & 62.18 & 86.98 & 70.97 \\
& \quad K-Center & 75.61 & 75.82 & 86.51 & 84.29 & 90.80 & 88.08 & 36.16 & 32.25 & 42.05 & 39.09 & 54.19 & 53.89 & 68.91 & 64.72 & 80.99 & 72.73 & 80.14 & 75.07 \\
& \quad FL-RBF   & 88.00 & 81.83        & 90.48 & 87.01 & 92.13 & 89.22 & 55.56 & 50.13  & 62.11 & 56.99 & 67.41 & 63.09 & 85.52 & 68.36 & 87.71 & 73.25 & 88.77 & 78.32 \\
\cmidrule{2-20}

& \textbf{SGS (Ours)} & & & & & & & & & & & & & & & & & & \\
& \quad FL-RBF (k=0.4)
& 88.33 & 83.96 & 90.72 & 87.37 & 91.89 & 89.42
& 57.84 & 53.47 & 64.08 & 60.09 & 69.18 & 64.96
& 84.95 & 74.89 & 87.41 & 77.90 & 88.72 & 81.61 \\
& \quad \dlt{$\Delta$ vs FL-RBF}
& \dlt{+0.33} & \dlt{+2.13} & \dlt{+0.24} & \dlt{+0.36} & \dlt{-0.24} & \dlt{+0.20}
& \dlt{+2.28} & \dlt{+3.34} & \dlt{+1.97} & \dlt{+3.10} & \dlt{+1.77} & \dlt{+1.87}
& \dlt{-0.57} & \dlt{+6.53} & \dlt{-0.30} & \dlt{+4.65} & \dlt{-0.05} & \dlt{+3.29} \\
\addlinespace[1pt]

& \quad FL-RBF (k=0.6)
& 87.39 & 83.64 & 90.56 & 87.40 & 91.77 & 89.92
& 56.83 & 53.76 & 62.73 & 59.44 & 69.70 & 66.44
& 84.27 & 76.49 & 85.89 & 80.30 & 88.23 & 82.65 \\
& \quad \dlt{$\Delta$ vs FL-RBF}
& \dlt{-0.61} & \dlt{+1.81} & \dlt{+0.08} & \dlt{+0.39} & \dlt{-0.36} & \dlt{+0.70}
& \dlt{+1.27} & \dlt{+3.63} & \dlt{+0.62} & \dlt{+2.45} & \dlt{+2.29} & \dlt{+3.35}
& \dlt{-1.25} & \dlt{+8.13} & \dlt{-1.82} & \dlt{+7.05} & \dlt{-0.54} & \dlt{+4.33} \\

\bottomrule
\end{tabular}%
}
\vspace{-1.0em}
\end{table*}

\subsection{Main Results}
Fig.~\ref{fig:main_experiment} demonstrates the superiority of our \our{} (based on FL-RBF) compared to other pruning methods. The consistently best performance on both OA and mAcc, especially, mAcc values, illustrates the superiority of our methods over other baselines, owing to the established shared principles. We next separately analyze their effects in detail.

\textit{\textbf{Effectiveness of Term (B) Principles.}} Term (B) principles are realized through Calibrated Soft Labels (CSL) combined with Embedding Geometry Distillation (EGD). Fig.~\ref{fig:term B} demonstrates the gradual improvement achieved by these two methods. Improvements in both mAcc and OA can be observed after adding CSL, with further gains when incorporating EGD. This confirms our analysis that an improved structural likelihood can benefit across different priors. There is also a generally decaying trend when both EGD and CSL are applied under increasing budgets. This is because EGD primarily helps to mitigate the data scarcity issue by transferring information about the internal structure of the teacher's embedding; as more samples are selected, its effect weakens.

\textit{\textbf{Effectiveness of Term (A) Principles.}} First, we isolate the impact of the selection signal.  By applying our EGD and CSL framework across all baselines upon PointNet++ (Tab.~\ref{tab:fair comparison}), we find that naive global selection using \textbf{embedding signal} consistently outperforms classifier-derived scalar signals, confirming the audit in Tab.~\ref{tab:signal_audit}. These scalar signals are less class-comparable, therefore the selection is essential broken.  Crucially, since scalar signals correlate heavily with class size, methods that rely on them to estimate
\begin{wrapfigure}{l}{0.5\linewidth}
\centering
\vspace{-16pt}
\includegraphics[width=\linewidth]{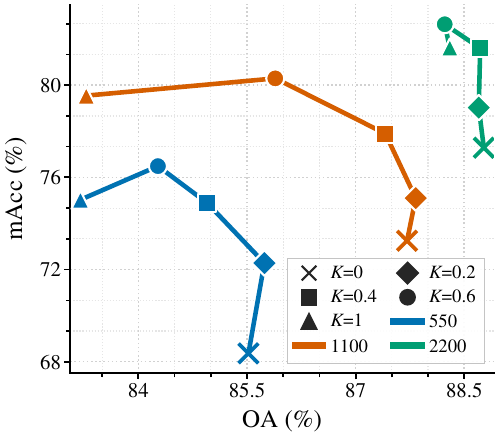}
\vspace{-1.7em}
\caption{Preferences steering with different K.}
\vspace{-20pt}
\label{fig:trade-off}
\end{wrapfigure}
``difficulty" (e.g., DRoP, NUCS, CCS-CP) can improve results yet still unreliable under such setting.
On ShapeNet55, these complex allocators are outperformed even by simple stratified sampling.
This validates our first principle: in the presence of strong inductive bias but weak scalar calibration, \emph{what you measure} (embedding geometry) matters more than \emph{how you allocate} based on noisy metrics.
Second, we analyze the necessity of the safety floor by varying the steering parameter $k$ in SGS (Fig.~\ref{fig:trade-off}).
The trajectory reveals two distinct regimes, validating our theory.

\textbf{(i)} Shared High-Reward Regime ($k \le 0.2$): Increasing the floor from $k=0$ (pure global) significantly boosts mAcc with negligible or no cost to OA. This confirms our second principle: a minimum floor is beneficial across different target priors. \\
\textbf{(ii)} Trade-off Regime ($k > 0.2$): Beyond the safety floor, a Pareto front emerges, allowing users to exchange OA for mAcc according to preference.
Notably, the pure mechanical balance ($k=1$) is often Pareto-dominated, reflecting its inability to adapt to the varying class complexities $c_y$ inherent in the Therm~\ref{thm:opt_alloc} (optimal allocation).

\subsection{Additional Results}
\textbf{Cross-Architecture Transfer.} We further validate whether our shared principles hold under cross-architecture pruning, where the teacher and student models have different architectures. Since EGD focuses solely on likelihood structure, we adopt RKD, which is free from embedding dimension constraints. Results in Fig.~\ref{fig:cross} demonstrate that cross-architecture pruning and training can be achieved with limited data (only 1,100 images). Notably, the student's performance is not hindered by architectural differences and can even benefit from a stronger teacher. For instance, PointNet++ and PointNext show improved mAcc when using PointVector as the teacher, highlighting the effectiveness of our principles in a broader setting.

\begin{figure}[tbp]
\centering
\includegraphics[width=0.93\linewidth]{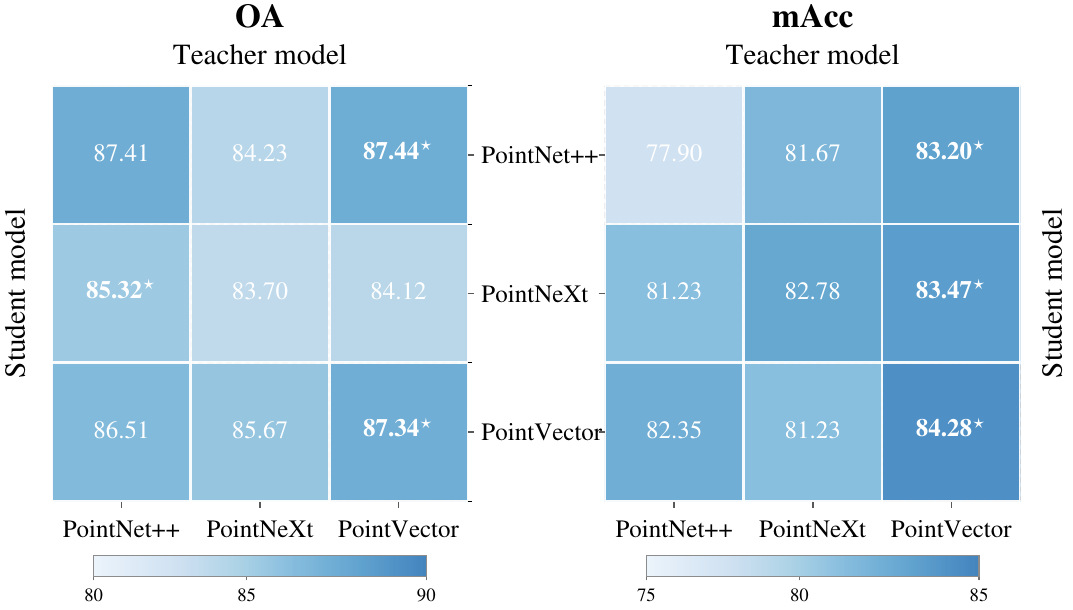}
\caption{Transferability of \our{} across different architecture. Results are based on $m=1100$ on ShapeNet55.}
\label{fig:cross}
\vspace{-15pt}
\end{figure}

\textbf{Generalization to Other Modalities.} Finally, we extend beyond the point cloud modality and show that our strategy can transfer to other modalities. Fig.~\ref{fig:meshnet} shows the results on the MeshNet with \our{} leading both OA and mAcc, validating that the effectiveness of our shared principles is not limited to point cloud, showing the potential to extend to other modalities as well.

\vspace{-0.05in}
\section{Conclusion}
In this work, we investigated 3D data pruning under severe and inherent class imbalance, where the gap between overall accuracy and mean accuracy makes pruning especially challenging. Instead of treating this as a purely unavoidable trade-off, we used a quadrature-based error decomposition to uncover shared optimization directions that benefit both evaluation priors. Based on this insight, we proposed \our{}, which first improves the performance floor by reducing prior-mismatch through calibrated geometric distillation and lowering representation error via safety-aware selection, and then handles the remaining preference gap with a lightweight steering wrapper. Extensive experiments show that this principled {\em foundation first, steering second} strategy consistently delivers strong and flexible performance across diverse architectures and evaluation settings.

\section*{Impact Statement}

This paper presents work whose goal is to advance the field of Machine Learning. There are many potential societal consequences of our work, none of which we feel must be specifically highlighted here.

\bibliography{example_paper}
\bibliographystyle{icml2026}

%%%%%%%%%%%%%%%%%%%%%%%%%%%%%%%%%%%%%%%%%%%%%%%%%%%%%%%%%%%%%%%%%%%%%%%%%%%%%%%
% APPENDIX
%%%%%%%%%%%%%%%%%%%%%%%%%%%%%%%%%%%%%%%%%%%%%%%%%%%%%%%%%%%%%%%%%%%%%%%%%%%%%%%
\newpage
\appendix
\onecolumn
\section*{Appendix}

\textbf{Roadmap.} This appendix provides supporting materials organized as follows:
\begin{itemize}
\item \textbf{Section~\ref{proof}}: Mathematical proofs for the generalization gap, class-wise error decomposition, Rademacher complexity bounds, optimal budget allocation, and distillation robustness.
\item \textbf{Section~\ref{additional experiment}}: Detailed experimental settings, including model architectures for PointNeXt and MeshMAE, and the three-phase training protocols.
\item \textbf{Section~\ref{additional results}}: Extended empirical evaluations across diverse 3D architectures and performance analysis under relaxed data pruning budgets.
\end{itemize}

\section{Proof} \label{proof}
\subsection{Proof for Lemma~\ref{lem:gen_gap}}
\begin{proof}
Let $q = q_{S,w}$ for brevity and $D := D_{\mathcal{G}}(p, q)$.
For any $\theta$,
\[
\mathcal{L}(\theta) \leq \hat{\mathcal{L}}_q(\theta) + D \quad \text{and} \quad \hat{\mathcal{L}}_q(\theta) \leq \mathcal{L}(\theta) + D
\]
by the definition of $D$ (taking $g = \ell_\theta$). Since $\hat{\theta}$ minimizes $\hat{\mathcal{L}}_q$,
\[
\hat{\mathcal{L}}_q(\hat{\theta}) \leq \hat{\mathcal{L}}_q(\theta^*).
\]
Therefore,
\[
\mathcal{L}(\hat{\theta}) \leq \hat{\mathcal{L}}_q(\hat{\theta}) + D \leq \hat{\mathcal{L}}_q(\theta^*) + D \leq \mathcal{L}(\theta^*) + 2D,
\]
which proves the claim.
\end{proof}

\subsection{Proof for Class-wise Decomposition}
\begin{proof}
Expand:
\begin{align*}
\mathbb{E}_p[\ell_\theta] - \mathbb{E}_q[\ell_\theta] &= \sum_y \pi_y^{\mathrm{tar}} \mathbb{E}_{p_y}[\ell_\theta] - \sum_y \rho_y \mathbb{E}_{q_y}[\ell_\theta] \\
&= \sum_y \pi_y^{\mathrm{tar}} \left( \mathbb{E}_{p_y}[\ell_\theta] - \mathbb{E}_{q_y}[\ell_\theta] \right) + \sum_y (\pi_y^{\mathrm{tar}} - \rho_y) \mathbb{E}_{q_y}[\ell_\theta].
\end{align*}

Take absolute values and apply triangle inequality:
\[
\leq \sum_y \pi_y^{\mathrm{tar}} \left| \mathbb{E}_{p_y}[\ell_\theta] - \mathbb{E}_{q_y}[\ell_\theta] \right| + \sum_y \left| \pi_y^{\mathrm{tar}} - \rho_y \right| \cdot \left| \mathbb{E}_{q_y}[\ell_\theta] \right|.
\]

Since $|\ell_\theta| \leq B$, we have $|\mathbb{E}_{q_y}[\ell_\theta]| \leq B$. Thus
\[
\leq \sum_y \pi_y^{\mathrm{tar}} \left| \mathbb{E}_{p_y}[\ell_\theta] - \mathbb{E}_{q_y}[\ell_\theta] \right| + B \sum_y \left| \pi_y^{\mathrm{tar}} - \rho_y \right| = \sum_y \pi_y^{\mathrm{tar}} \cdot |\cdot| + 2B \| \pi^{\mathrm{tar}} - \rho \|_{\mathrm{TV}}.
\]

The uniform version follows by taking $\sup_\theta$ on both sides.
\end{proof}

\subsection{Proof for Lemma~\ref{lem:class_rate}}
\begin{lemma}[Class-wise error bound (formal)]
\label{lem:classwise_bound_app}
Fix a class $y$. Let $S_y=\{x^{(y)}_1,\ldots,x^{(y)}_{m_y}\}$ be drawn i.i.d.\ from the
class-conditional distribution $p_y$, and define the empirical measure
\[
q_y := \frac{1}{m_y}\sum_{i=1}^{m_y}\delta_{x^{(y)}_i}.
\]
Assume the loss is bounded: $0 \le \ell_\theta(x) \le B$ for all $\theta\in\Theta$ and $x$.
Let $\mathcal{G}_y := \{\ell_\theta(\cdot): \theta\in\Theta\}$ and define the empirical
Rademacher complexity
\[
\widehat{\mathfrak{R}}_{S_y}(\mathcal{G}_y)
:= \mathbb{E}_{\sigma}\left[\sup_{g\in\mathcal{G}_y}\frac{1}{m_y}\sum_{i=1}^{m_y}\sigma_i g\!\left(x^{(y)}_i\right)\right],
\]
where $\sigma_1,\ldots,\sigma_{m_y}$ are i.i.d.\ Rademacher variables.
Then for any $\delta\in(0,1)$, with probability at least $1-\delta$ over $S_y$,
\begin{equation}
\sup_{\theta\in\Theta}\left|\mathbb{E}_{p_y}[\ell_\theta]-\mathbb{E}_{q_y}[\ell_\theta]\right|
\;\le\;
2\,\widehat{\mathfrak{R}}_{S_y}(\mathcal{G}_y)
\;+\;
3B\sqrt{\frac{\log(4/\delta)}{2m_y}}.
\label{eq:classwise_bound_app}
\end{equation}
Moreover, if (for the considered hypothesis class) there exists a constant $c_y>0$ such that
\begin{equation}
\widehat{\mathfrak{R}}_{S_y}(\mathcal{G}_y)\le \frac{c_y}{2\sqrt{m_y}},
\label{eq:rademacher_cy_app}
\end{equation}
then with the same probability,
\begin{equation}
\sup_{\theta\in\Theta}\left|\mathbb{E}_{p_y}[\ell_\theta]-\mathbb{E}_{q_y}[\ell_\theta]\right|
\;\le\;
\frac{c_y}{\sqrt{m_y}}
\;+\;
3B\sqrt{\frac{\log(4/\delta)}{2m_y}}.
\label{eq:classwise_cy_app}
\end{equation}
\end{lemma}

\begin{proof}
Define the normalized loss class
\[
\widetilde{\mathcal{G}}_y := \left\{x \mapsto \frac{\ell_\theta(x)}{B} : \theta\in\Theta\right\}.
\]
Since $0\le \ell_\theta(x)\le B$, every function in $\widetilde{\mathcal{G}}_y$ maps into $[0,1]$.
Apply the Rademacher-complexity generalization bound (Theorem 3.3 in \citep{mohri2018foundations})
to the class $\widetilde{\mathcal{G}}_y$ with confidence parameter $\delta/2$. With probability
at least $1-\delta/2$, the following holds simultaneously for all $\theta\in\Theta$:
\begin{equation}
\mathbb{E}_{p_y}\!\left[\frac{\ell_\theta}{B}\right]
\le
\mathbb{E}_{q_y}\!\left[\frac{\ell_\theta}{B}\right]
+
2\,\widehat{\mathfrak{R}}_{S_y}(\widetilde{\mathcal{G}}_y)
+
3\sqrt{\frac{\log(4/\delta)}{2m_y}}.
\label{eq:one_side_upper_app}
\end{equation}
To obtain the reverse inequality, apply the same theorem to the class
$1-\widetilde{\mathcal{G}}_y := \{x\mapsto 1-g(x): g\in \widetilde{\mathcal{G}}_y\}$
(which also maps into $[0,1]$) with confidence parameter $\delta/2$. With probability at least
$1-\delta/2$, for all $\theta\in\Theta$,
\[
\mathbb{E}_{p_y}\!\left[1-\frac{\ell_\theta}{B}\right]
\le
\mathbb{E}_{q_y}\!\left[1-\frac{\ell_\theta}{B}\right]
+
2\,\widehat{\mathfrak{R}}_{S_y}(1-\widetilde{\mathcal{G}}_y)
+
3\sqrt{\frac{\log(4/\delta)}{2m_y}}.
\]
Rearranging gives
\begin{equation}
\mathbb{E}_{p_y}\!\left[\frac{\ell_\theta}{B}\right]
\ge
\mathbb{E}_{q_y}\!\left[\frac{\ell_\theta}{B}\right]
-
2\,\widehat{\mathfrak{R}}_{S_y}(1-\widetilde{\mathcal{G}}_y)
-
3\sqrt{\frac{\log(4/\delta)}{2m_y}}.
\label{eq:one_side_lower_app}
\end{equation}
By a union bound over the two events, with probability at least $1-\delta$, both
\eqref{eq:one_side_upper_app} and \eqref{eq:one_side_lower_app} hold simultaneously for all
$\theta$. Noting that
$\widehat{\mathfrak{R}}_{S_y}(1-\widetilde{\mathcal{G}}_y)=\widehat{\mathfrak{R}}_{S_y}(\widetilde{\mathcal{G}}_y)$
(translation by a constant does not change Rademacher complexity, and sign-flips are absorbed by
Rademacher symmetry), we obtain
\[
\sup_{\theta\in\Theta}\left|
\mathbb{E}_{p_y}\!\left[\frac{\ell_\theta}{B}\right]
-
\mathbb{E}_{q_y}\!\left[\frac{\ell_\theta}{B}\right]
\right|
\le
2\,\widehat{\mathfrak{R}}_{S_y}(\widetilde{\mathcal{G}}_y)
+
3\sqrt{\frac{\log(4/\delta)}{2m_y}}.
\]
Multiplying by $B$ yields \eqref{eq:classwise_bound_app} and using
$\widehat{\mathfrak{R}}_{S_y}(\widetilde{\mathcal{G}}_y) = \widehat{\mathfrak{R}}_{S_y}(\mathcal{G}_y)/B$
gives the form stated in the lemma.

Finally, condition \eqref{eq:rademacher_cy_app} is satisfied by many hypothesis classes.
In particular, for spectrally-controlled neural networks, the scale-sensitive complexity depends on
data norm and products of spectral norms (see, e.g., \cite{bartlett2017spectrally}); absorbing the
resulting constants/log factors into $c_y$ yields \eqref{eq:rademacher_cy_app}, and plugging it into
\eqref{eq:classwise_bound_app} gives \eqref{eq:classwise_cy_app}.
\end{proof}

% ------------------------------------------------------------
\subsection{Proof for Theorem~\ref{thm:opt_alloc}}
% ------------------------------------------------------------

\begin{theorem}[Optimal allocation for Term (A)]
\label{thm:alloc_app}
Let $\gamma>0$ be a common exponent across classes, and let $\pi_y^{\mathrm{tar}}>0$, $c_y>0$.
Consider the continuous relaxation of the budget allocation problem:
\[
\min_{\{m_y>0\}}\ \sum_y \pi_y^{\mathrm{tar}}\frac{c_y}{m_y^\gamma}
\quad \text{s.t.}\quad \sum_y m_y = m,
\]
where $m>0$ is the total budget. Then the unique optimizer satisfies
\begin{equation}
m_y^* \;=\; m\cdot
\frac{(\pi_y^{\mathrm{tar}}c_y)^{\frac{1}{1+\gamma}}}{\sum_{y'}(\pi_{y'}^{\mathrm{tar}}c_{y'})^{\frac{1}{1+\gamma}}}
\quad\Longleftrightarrow\quad
m_y^* \propto (\pi_y^{\mathrm{tar}}c_y)^{\frac{1}{1+\gamma}}.
\label{eq:alloc_app}
\end{equation}
\end{theorem}

\begin{proof}
Define $a_y := \pi_y^{\mathrm{tar}}c_y>0$ and consider
\[
\min_{\{m_y>0\}}\ \sum_y a_y m_y^{-\gamma}\quad \text{s.t.}\quad \sum_y m_y = m.
\]
Form the Lagrangian
\[
\mathcal{J}(m_1,\ldots,m_C,\lambda)
=
\sum_y a_y m_y^{-\gamma} + \lambda\Big(\sum_y m_y - m\Big).
\]
Taking derivatives and setting them to zero yields, for each class $y$,
\[
\frac{\partial \mathcal{J}}{\partial m_y}
=
-\gamma a_y m_y^{-(\gamma+1)} + \lambda
=0
\quad\Longrightarrow\quad
m_y^{\gamma+1} = \frac{\gamma a_y}{\lambda}.
\]
Thus $m_y \propto a_y^{1/(1+\gamma)}$. Enforcing the constraint $\sum_y m_y=m$ gives
\[
m_y^* = m\cdot \frac{a_y^{1/(1+\gamma)}}{\sum_{y'} a_{y'}^{1/(1+\gamma)}}
=
m\cdot
\frac{(\pi_y^{\mathrm{tar}}c_y)^{\frac{1}{1+\gamma}}}{\sum_{y'}(\pi_{y'}^{\mathrm{tar}}c_{y'})^{\frac{1}{1+\gamma}}}.
\]
This is the stated solution \eqref{eq:alloc_app}.
\end{proof}

\subsection{Proof of Proposition~\ref{prop:kd_weight_robust}}
% ------------------------------------------------------------
% Appendix: Proof of Proposition~\ref{prop:kd_weight_robust}
% ------------------------------------------------------------
\begin{proof}
Fix a subset $S$ and weights $w$ with $w_i>0$ for all $i\in S$.
For brevity write
\[
T_i(\cdot) := T(\cdot\mid x_i)\in\Delta^{|\mathcal{Y}|-1},
\qquad
f_i(\cdot) := f_\theta(\cdot\mid x_i)\in\Delta^{|\mathcal{Y}|-1}.
\]
Recall that the (categorical) cross-entropy is
\[
\mathrm{CE}(T_i,f_i) \;:=\; -\sum_{y\in\mathcal{Y}} T_i(y)\,\log f_i(y),
\]
with the usual convention that if $T_i(y)>0$ and $f_i(y)=0$ then
$\mathrm{CE}(T_i,f_i)=+\infty$.

Using the standard decomposition of cross-entropy into entropy plus KL divergence,
for each $i\in S$ we have
\begin{align*}
\mathrm{CE}(T_i,f_i)
&= -\sum_{y} T_i(y)\log f_i(y) \\
&= -\sum_{y} T_i(y)\log T_i(y) \;+\; \sum_{y} T_i(y)\log\frac{T_i(y)}{f_i(y)} \\
&=: H(T_i) \;+\; \mathrm{KL}\!\left(T_i\,\|\,f_i\right),
\end{align*}
where $H(T_i)$ is the Shannon entropy and $\mathrm{KL}(\cdot\|\cdot)$ is the
Kullback--Leibler divergence.
By Gibbs' inequality, $\mathrm{KL}(T_i\|f_i)\ge 0$, with equality if and only if
$f_i(\cdot)=T_i(\cdot)$ (as distributions on $\mathcal{Y}$).

Plugging this into the weighted KD objective~\eqref{eq:kd_emp_main} yields
\begin{align*}
\hat{\mathcal{L}}_{\mathrm{KD}}^{S,w}(\theta)
&= \sum_{i\in S} w_i\,\mathrm{CE}(T_i,f_i) \\
&= \sum_{i\in S} w_i\,H(T_i)
\;+\;
\sum_{i\in S} w_i\,\mathrm{KL}\!\left(T_i\,\|\,f_i\right).
\end{align*}
The first term $\sum_{i\in S} w_i H(T_i)$ does not depend on $\theta$.
The second term is a weighted sum of nonnegative quantities, hence
\[
\hat{\mathcal{L}}_{\mathrm{KD}}^{S,w}(\theta)
\;\ge\;
\sum_{i\in S} w_i\,H(T_i),
\]
and equality holds if and only if $\mathrm{KL}(T_i\|f_i)=0$ for every $i\in S$,
i.e., if and only if $f_\theta(\cdot\mid x_i)=T(\cdot\mid x_i)$ for all $i\in S$.

By assumption, the student is expressive enough to interpolate the teacher on $S$,
so there exists at least one $\theta$ satisfying
$f_\theta(\cdot\mid x_i)=T(\cdot\mid x_i)$ for all $i\in S$.
For any such interpolating $\theta$, the KL terms vanish and the lower bound is
achieved; therefore every interpolating solution is a global minimizer of
$\hat{\mathcal{L}}_{\mathrm{KD}}^{S,w}$.
Since the argument holds for any choice of strictly positive weights $w$,
the claim follows.
\end{proof}

\section{Additional Experimental Setup} \label{additional experiment}

\subsection{Model Architecture}
We use PointNeXt-S and PointVector-S. For MeshMAE and PointMAE we use their default architecture when processing corresponding datasets. The PointMAE and MeshMAE model directly reuse their official checkpoints. We present the performance of model trained by us on different datasets in Tab~

\subsection{Training Strategy}\label{sec:training_strategy}

Our training pipeline is conducted in three phases, all using a fixed random seed of 42. We strictly follow the data augmentation protocols outlined in the original papers for each model.

\begin{table}[!ht]
\centering
\small
\caption{Cross-dataset evaluation across 3D backbones.
Overall accuracy (OA) and mean accuracy (mAcc) on ShapeNet55, ModelNet40, and ScanObjectNN.}
\label{tab:cross_dataset_backbones}
\vspace{-0.5em}
\setlength{\tabcolsep}{3.5pt}
\renewcommand{\arraystretch}{1.05}
\begin{tabular}{@{}lcccccc@{}}
\toprule
& \multicolumn{2}{c}{ShapeNet55} & \multicolumn{2}{c}{ModelNet40} & \multicolumn{2}{c}{ScanObjectNN} \\
\cmidrule(lr){2-3}\cmidrule(lr){4-5}\cmidrule(lr){6-7}
Model & OA & mAcc & OA & mAcc & OA & mAcc \\
\midrule
PointNeXt-S     & 91.11 & 83.47 & 93.15 & 90.60 & 87.20 & 85.80 \\
PointNet++      & 91.08 & 82.83 & 92.95 & 90.34 & 86.12 & 84.37 \\
PointVector-S   & 91.00 & 82.96 & 90.15 & 85.31 & 87.37 & 85.64 \\
\bottomrule
\end{tabular}
\end{table}

\paragraph{Initial Model Training.}
The base model is trained for 300 epochs with a batch size of 48. We utilize the AdamW optimizer with an initial learning rate of $1 \times 10^{-4}$ and a weight decay of $0.05$. The learning rate is managed by a cosine annealing scheduler with 2 epochs of linear warmup. Notably, we apply label smoothing of $0.2$ throughout this phase.

\paragraph{Class-Balanced Classifier Retraining.}
In this phase, we freeze the encoder and retrain the reinitialized classification head for 100 epochs. We employ a UniformClassSampler and mixed-precision (bfloat16) training. The learning rate is adjusted to $1.7 \times 10^{-4}$, while the weight decay remains $0.05$ and the warmup period stays at 2 epochs. Consistent with the initial phase, the label smoothing parameter is maintained at $0.2$.

\paragraph{Knowledge Distillation.}
The student model is trained from scratch on the pruned dataset for 300 epochs, reusing the optimizer configuration from the initial phase. For the distillation parameters, we set the temperature $\tau=5$ and the knowledge distillation weight $\alpha=0.8$. Regarding Relational Knowledge Distillation (RKD), we use a distance weight $\lambda_d=50$ and an angle weight $\lambda_a=100$, applied with an overall scaling factor of $\lambda=0.1$.

\section{Additional Results} \label{additional results}
\subsection{Additional Results on Different Architectures}
Tab.~\ref{tab:pointnext_pointvector} presents results on PointNeXt and PointVector, demonstrating that our method generalizes beyond the architectures in the main paper. Tab.~\ref{tab: MeshMAE} shows results on MeshMAE, a more modern mesh-based architecture, further confirming the robustness of our approach. These findings support the broad effectiveness of our proposed shared principles.

\begin{figure}[t]
\centering
\includegraphics[width=0.44\linewidth]{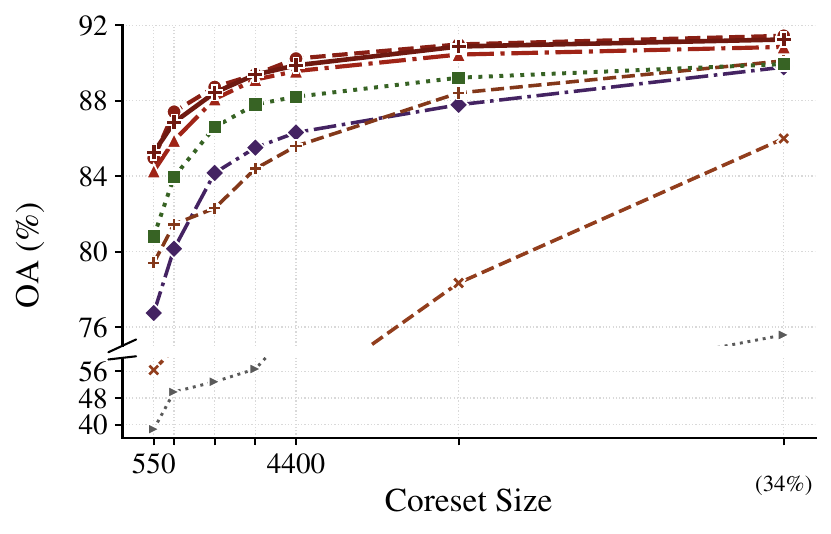}
\includegraphics[width=0.44\linewidth]{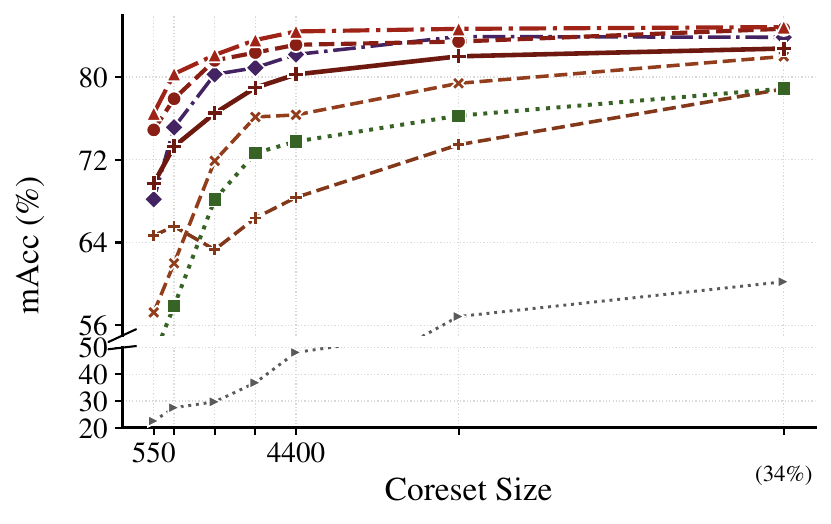}

\includegraphics[width=0.75\linewidth]{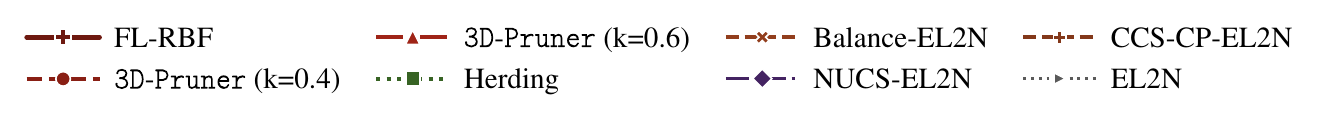}
\caption{Additional results on relaxed budgets on ShapeNet55. The gap on OA becomes smaller given larger budget, yet our \our{} remains the best on the mAcc.}
\label{fig:relaxed budget}
\end{figure}

\begin{table}[H]
\centering
\caption{Comparison of different pruning methods on ModelNet40, MeshMAE. The pruning methods are selected from the most competitive ones based on the performance on the MeshNet.}
\label{tab: MeshMAE}
\resizebox{0.56\textwidth}{!}{%
\begin{tabular}{ll cc cc cc}
\toprule
& & \multicolumn{6}{c}{\textbf{ModelNet40}} \\
& & \multicolumn{2}{c}{$m=400$} & \multicolumn{2}{c}{$m=800$} & \multicolumn{2}{c}{$m=1200$} \\
\cmidrule(lr){3-4} \cmidrule(lr){5-6} \cmidrule(lr){7-8}
\textbf{Metric Type} & \textbf{Selection Policy} & OA & mAcc & OA & mAcc & OA & mAcc \\
\midrule
% -------------------------------------------
% SCALAR SCORES
% -------------------------------------------
\multirow{4}{*}{\shortstack{\textbf{Scalar}\\\textbf{Scores}}}
& \textit{Stratified Sampling} & & & & & & \\
& \quad Balance-EL2N        & 60.60 & 56.71 & 76.83 & 74.72 & 82.85 & 80.68 \\
& \quad NUCS-EL2N        & 57.06 & 52.99 & 72.89 & 70.68 & 81.27 & 78.40 \\
& \quad CCS-CP-EL2N      & 68.25 & 57.17 & 78.83 & 70.07 & 83.12 & 75.21 \\
\midrule
\midrule
% -------------------------------------------
% VECTOR EMBEDDINGS
% -------------------------------------------
\multirow{7}{*}{\shortstack{\textbf{Vector}\\\textbf{Emb.}}}
& \textit{Global Selection} & & & & & & \\
& \quad K-Center & 59.00 & 58.16 & 79.14 & 74.63 & 83.17 & 79.45 \\
& \quad FL-RBF   & 76.45 & 68.71 & 85.43 & 80.64 & 87.95 & 84.05 \\
\cmidrule{2-8}
& \textbf{SGS (Ours)} & & & & & & \\
& \quad FL-RBF (k=0.4)
& 75.87 & 70.50 & 85.45 & 80.68 & 88.27 & 85.01 \\
& \quad \dlt{$\Delta$ vs FL-RBF}
& \dlt{-0.58} & \dlt{+1.79} & \dlt{+0.02} & \dlt{+0.04} & \dlt{+0.32} & \dlt{+0.96} \\
\addlinespace[1pt]
& \quad FL-RBF (k=0.6)
& 76.85 & 71.19 & 85.87 & 80.89 & 88.19 & 85.22 \\
& \quad \dlt{$\Delta$ vs FL-RBF}
& \dlt{+0.40} & \dlt{+2.48} & \dlt{+0.44} & \dlt{+0.25} & \dlt{+0.24} & \dlt{+1.17} \\
\bottomrule
\end{tabular}
}
\end{table}

\subsection{Additional Results on Relaxed Budgets}
Fig.~\ref{fig:relaxed budget} presents comparisons under more relaxed budgets (retaining up to 40\% of the original data). We select the best-performing pruning methods and further increase the budget to compare them. Results validate the effectiveness of \our{} at larger pruning ratios; nevertheless, we observe that the preference steering induced by different $K$ values diminish under higher budgets, particularly at 34\%. As analyzed in the main paper, ShapeNet55 contains fewer than 5\% few-shot cases. However, larger budgets increase the number of many-shot classes, making the prior mismatch term more dominant. This explains why our method remains competitive \textit{\textbf{on the mAcc}} under this setting, as it mitigates this effect through calibrated teacher soft-labels and geometric embedding distillation.
\begin{table}[H]
\centering
\caption{Additional results on PointNeXt and PointVector.}
\vspace{-0.5em}
\label{tab:pointnext_pointvector}
\resizebox{0.88 \textwidth}{!}{
\begin{tabular}{l cc cc cc cc cc cc}
\toprule
& \multicolumn{6}{c}{\textbf{PointNeXt-S}} & \multicolumn{6}{c}{\textbf{PointVector-S}} \\
& \multicolumn{2}{c}{$m=550$} & \multicolumn{2}{c}{$m=1100$} & \multicolumn{2}{c}{$m=2200$}
& \multicolumn{2}{c}{$m=550$} & \multicolumn{2}{c}{$m=1100$} & \multicolumn{2}{c}{$m=2200$} \\
\cmidrule(lr){2-3} \cmidrule(lr){4-5} \cmidrule(lr){6-7}
\cmidrule(lr){8-9} \cmidrule(lr){10-11} \cmidrule(lr){12-13}
\textbf{Method} & OA & mAcc & OA & mAcc & OA & mAcc & OA & mAcc & OA & mAcc & OA & mAcc \\
\midrule
Loss     & 71.32 & 31.82 & 75.90 & 41.46 & 78.38 & 40.01 & 71.75 & 28.02 & 74.71 & 34.17 & 76.35 & 35.37 \\
GradNorm & 76.75 & 37.96 & 79.36 & 43.69 & 82.25 & 53.49 & 80.25 & 49.30 & 82.76 & 56.60 & 84.31 & 60.24 \\
Entropy  & 72.01 & 30.96 & 74.98 & 34.71 & 78.91 & 41.62 & 71.99 & 31.82 & 75.93 & 38.64 & 78.05 & 42.24 \\
EL2N     & 71.71 & 31.69 & 76.57 & 36.99 & 78.04 & 41.45 & 70.47 & 26.15 & 74.42 & 30.29 & 78.29 & 37.66 \\
DRoP         & 50.11 & 59.31 & 59.05 & 68.10 & 64.36 & 72.57 & 42.30 & 56.77 & 52.87 & 63.24 & 59.56 & 69.99 \\
Balance-EL2N        & 80.51 & 74.92 & 83.55 & 78.59 & 85.90 & 82.18 & 80.23 & 75.34 & 82.11 & 76.56 & 84.92 & 80.22 \\
NUCS-EL2N        & 78.01 & 71.94 & 81.76 & 76.68 & 85.48 & 80.94 & 77.63 & 71.99 & 81.53 & 77.58 & 84.60 & 80.56 \\
CCS-CP-EL2N      & 79.80 & 52.77 & 82.40 & 59.20 & 83.90 & 59.74 & 79.44 & 53.41 & 81.56 & 57.67 & 83.96 & 62.23 \\
Herding  & 83.05 & 59.61 & 85.06 & 64.00 & 87.44 & 72.21 & 81.93 & 57.13 & 84.97 & 67.28 & 87.10 & 71.25 \\
K-Center & 59.77 & 60.96 & 76.38 & 69.32 & 85.99 & 76.04 & 56.24 & 62.95 & 69.47 & 67.99 & 84.21 & 76.92 \\
FL-RBF   & 82.51 & 74.35 & 83.45 & 77.78 & 86.10 & 80.92 & 83.13 & 74.20 & 85.41 & 76.45 & 86.20 & 80.32 \\
\midrule
\textbf{\our{} (Ours)} & & & & & & & & & & & & \\
\quad FL-RBF (k=0.4) & \textbf{85.50} & 75.75 & \textbf{86.80} & 78.80 & \textbf{88.72} & 81.02 & \textbf{84.78} & 76.11 & \textbf{87.41} & 79.48 & \textbf{88.91} & 82.96 \\
\quad FL-RBF (k=0.6) & 83.68 & \textbf{76.51} & 86.04 & \textbf{80.54} & 88.35 & \textbf{83.51} & 84.07 & \textbf{76.46} & 86.29 & \textbf{80.43} & 88.23 & \textbf{82.98} \\
\bottomrule
\end{tabular}
}
\end{table}

\end{document}